\documentclass[letterpaper, 10 pt, conference]{ieeeconf}  \IEEEoverridecommandlockouts                              
\let\labelindent\relax

\usepackage{xcolor}
\usepackage{colortbl}
\usepackage{tabu}
\usepackage{amsmath}
\usepackage{amssymb}
\usepackage{url}
\usepackage{resizegather}
\usepackage{graphicx,import}
\usepackage{cuted}
\usepackage[inline]{enumitem}
\usepackage{booktabs} \usepackage{mathtools}
\usepackage{array}
\newcolumntype{H}{>{\setbox0=\hbox\bgroup}c<{\egroup}@{}}
\usepackage{multirow}
\usepackage{xstring}
\usepackage{xparse}
\usepackage{soul}
\usepackage{cite}

\usepackage{tikz}
\usetikzlibrary{automata,positioning}
\usetikzlibrary{shapes,arrows}
\tikzset{state/.style={circle, draw, minimum size=0.5cm}}

\usepackage{algorithm}
\usepackage{algorithmicx}
\usepackage{algpseudocode}

\newlength{\continueindent}
\usepackage{etoolbox}
\makeatletter
\newcommand*{\ALG@customparshape}{\parshape 2 \leftmargin \linewidth \dimexpr\ALG@tlm+\continueindent\relax \dimexpr\linewidth+\leftmargin-\ALG@tlm-\continueindent\relax}
\apptocmd{\ALG@beginblock}{\ALG@customparshape}{}{\errmessage{failed to patch}}
\makeatother

\usepackage{cleveref}
\newtheorem{definition}{Definition}
\newtheorem{remark}{Remark}
\newtheorem{assumption}{Assumption}

\newtheorem{theorem}{Theorem}
\crefname{section}{Section}{Sections}
\crefname{subsection}{Section}{Sections}
\crefname{definition}{Definition}{Definitions}
\crefname{proposition}{Proposition}{Propositions}
\crefname{lemma}{Lemma}{Lemmas}
\crefname{theorem}{Theorem}{Theorems}
\crefname{corollary}{Corollary}{Corollaries}
\crefname{example}{Example}{Examples}
\crefname{figure}{Figure}{Figures}
\crefname{table}{Table}{Tables}
\crefname{assumption}{Assumption}{Assumptions}
\crefname{remark}{Remark}{Remarks}
\crefname{running}{Running Example}{Running Examples}
\crefname{algorithm}{Algorithm}{Algorithms}
\usepackage{xspace}

\newcommand{\pctls}{PCTL\xspace}

\newcommand{\nat}{\mathbb{N}}
\newcommand{\real}{\mathbb{R}}

\renewcommand{\epsilon}{\varepsilon}

\newcommand{\abs}[1]{\vert #1 \vert}

\newcommand{\m}{\mathcal{M}}

\newcommand{\ap}{\mathsf{a}}
\newcommand{\aps}{\mathrm{AP}}

\newcommand{\ms}{s}
\newcommand{\ma}{a}

\newcommand{\mss}{{S}}
\newcommand{\mas}{{A}}
\newcommand{\mts}{\mathbf{T}}
\newcommand{\mi}{s_\mathrm{init}}

\newcommand{\ml}{L}

\newcommand{\ltrue}{\mathtt{True}}
\newcommand{\lfalse}{\mathtt{False}}
\newcommand{\lnext}{\mathbf{X}}
\newcommand{\luntil}{\mathbf{U}}
\newcommand{\lrelease}{\mathbf{R}}
\newcommand{\lfinal}{\mathbf{F}}

\newcommand{\lglobal}{\mathbf{G}}

\newcommand{\lp}{\mathbf{P}}

\newcommand{\pr}{\mathbb{P}}

\definecolor{redish}		{rgb}{0.8, 0.1, 0.1}
\definecolor{blueish}		{rgb}{0  , 0  , 0.4}
\definecolor{greenish}		{rgb}{0  , 0.6, 0  }
\definecolor{yellowish}		{rgb}{0.8, 0.5, 0  }
\definecolor{redishBG}		{rgb}{1  , 0.75, 0.75}
\definecolor{blueishBG}		{rgb}{0.75, 0.75, 1  }
\definecolor{greenishBG}	{rgb}{0.75, 1  , 0.75}
\definecolor{yellowishBG}	{rgb}{1  , 0.7, 0  }

\title{\LARGE \bf
Statistically Model Checking PCTL Specifications on Markov Decision Processes via Reinforcement Learning
}

\author{Yu Wang, Nima Roohi, Matthew West, Mahesh Viswanathan, and Geir E. Dullerud\thanks{Yu Wang is with Duke University, USA
        {\tt\small yu.wang094@duke.edu}. 
        Nima Roohi is with the University of California San Diego, USA
        {\tt\small nroohi@ucsd.edu}. 
        Matthew West, Mahesh Viswanathan, and Geir E. Dullerud are with the University of Illinois at Urbana-Champaign, USA
        {\tt\small \{mwest, vmahesh, dullerud\}@illinois.edu}}}

\begin{document}

\maketitle
\thispagestyle{empty}
\pagestyle{empty}

\begin{abstract}
Probabilistic Computation Tree Logic (PCTL) is frequently used to formally specify control objectives such as probabilistic reachability and safety. In this work, we focus on model checking PCTL specifications statistically on Markov Decision Processes (MDPs) by sampling, e.g., checking whether there exists a feasible policy such that the probability of reaching certain goal states is greater than a threshold. We use reinforcement learning to search for such a feasible policy for PCTL specifications, and then develop a statistical model checking (SMC) method with provable guarantees on its error. Specifically, we first use upper-confidence-bound (UCB) based Q-learning to design an SMC algorithm for bounded-time PCTL specifications, and then extend this algorithm to unbounded-time specifications by identifying a proper truncation time by checking the PCTL specification and its negation at the same time. Finally, we evaluate the proposed method on case studies.
\end{abstract}

\section{Introduction} \label{sec:intro}

Probabilistic Computation Tree Logic (PCTL) is frequently used
to formally specify control objectives such as reachability and safety
on probabilistic systems~\cite{baier_ModelcheckingAlgorithmsContinuoustime_2003}.
To check the correctness of PCTL specifications on these systems,
\emph{model checking} methods are required~\cite{clarke_HandbookModelChecking_2018}.
Although model checking PCTL by model-based analysis 
is theoretically possible~\cite{baier_ModelcheckingAlgorithmsContinuoustime_2003},
it is not preferable in practice when the system model is unknown or large.
In these cases, model checking by sampling, i.e. statistical model checking (SMC), is needed~\cite{larsen_StatisticalModelChecking_2016,agha_SurveyStatisticalModel_2018}.

The statistical model checking of PCTL specifications 
on Markov Decision Processes (MDPs)
is frequently encountered in many decision problems
-- e.g., 
for a robot in a grid world under probabilistic disturbance,
checking whether there exists a feasible control policy such that 
the probability of reaching certain goal states 
is greater than a probability threshold~\cite{fainekos_TemporalLogicMotion_2009,kress-gazit_SynthesisRobotsGuarantees_2018,bozkurt_ControlSynthesisLinear_2020}. 
In these problems, the main challenge is to search for such a feasible policy
for the PCTL specification of interest.

To search for feasible policies for temporal logics specifications, such as PCTL, on MDPs, 
one approach is model-based reinforcement learning~\cite{henriques_StatisticalModelChecking_2012,brazdil_VerificationMarkovDecision_2014,fu_ProbablyApproximatelyCorrect_2014,ashok_PACStatisticalModel_2019} 
-- i.e., first inferring the transition probabilities of the MDP 
by sampling over each state-action pair, 
and then searching for the feasible policy via model-based analysis.
This approach is often inefficient, since not all transition probabilities
are relevant to the PCTL specification of interest.
Here instead, we adopt a model-free reinforcement learning approach~\cite{sutton_ReinforcementLearningIntroduction_2018}.

Common model-free reinforcement learning techniques cannot directly handle 
temporal logic specifications.
One solution is to find a surrogate reward function
such that the policy learned for this surrogate reward function 
is the one needed for checking the temporal logic specification of interest.
For certain temporal logics interpreted under special semantics (usually involving a metric),
the surrogate reward can be found based on that semantics~\cite{li_ReinforcementLearningTemporal_2016,jones_RobustSatisfactionTemporal_2015,littman_EnvironmentIndependentTaskSpecifications_2017}.

For temporal logics under the standard semantics~\cite{baier_PrinciplesModelChecking_2008},
the surrogate reward functions can be derived via constructing 
the product MDP~\cite{hahn_OmegaRegularObjectivesModelFree_2019a,hasanbeig_ReinforcementLearningTemporal_2019,bozkurt_ControlSynthesisLinear_2020}
of the initial MDP and 
the automaton realizing the temporal logic specification.
However, the complexity of constructing the automaton from a general linear temporal logic (LTL) specification 
is \emph{double exponential}
\cite{baier_PrinciplesModelChecking_2008,hahn_LazyProbabilisticModel_2013}.
For a fraction of LTL, namely LTL/GU, the complexity is  \emph{exponential}
\cite{DBLP:conf/fsttcs/KiniV17,kini2015limit}.
In addition, the size of the product MDP is usually much larger than the initial MDP, 
although the produce MDP may be constructed on-the-fly 
to reduce the extra computation cost, as it did in
\cite{hasanbeig_ReinforcementLearningTemporal_2019}.

In this work, we propose a new statistical model checking method 
for PCTL specifications on MDPs.
For a lucid discussion, 
we only consider non-nested PCTL specifications.
PCTL formulas in general form with nested
probabilistic operators can be handled in the standard manner using the approach proposed
in~\cite{sen_StatisticalModelChecking_2004,sen_StatisticalModelChecking_2005}.
Our method uses upper-confidence-bound (UCB) based Q-learning 
to directly learn  the feasible policy of PCTL specifications,
without constructing the product MDP.
The effectiveness of UCB-based Q-learning 
has been proven for the $K$-bandit problem, and has been numerically demonstrated on many decision-learning problems on MDPs 
(see~\cite{szepesvari_AlgorithmsReinforcementLearning_2010}).

For bounded-time PCTL specifications, 
we treat the statistical model checking
problem as a finite sequence of $K$-bandit problems and use 
the UCB-based Q-learning to learn the desirable decision
at each time step.
For unbounded-time PCTL specifications,
we look for a truncation time to reduce it to a 
bounded-time problem by checking the PCTL specification
and its negation at the same time.  
Our statistical model checking algorithm is \emph{online};
it terminates with probability 1, and only when the statistical error of the 
learning result is smaller than a user-specified value.

The rest of the paper is organized as follows. The preliminaries on labeled MDPs and PCTL are given in \cref{sec:prelim}.
In \cref{sec:reach_by_q}, using the principle of optimism in the face of uncertainty, we design Q-learning algorithms to solve finite-time and infinite-time probabilistic satisfaction, and give finite sample probabilistic guarantees for the correctness of the algorithms.
We implement and evaluate the proposed algorithms on several case studies in \cref{sec:simulation}.
Finally, we conclude this work in \cref{sec:conclusion}.

\section{Preliminaries and Problem Formulation}
\label{sec:prelim}

The set of integers and real numbers are denoted by $\nat$ and $\real$, respectively.
For $n \in \nat$, let $[n] = \{1,\ldots,n\}$.
The cardinality of a set is denoted by $\abs{\cdot}$.
The set of finite-length sequences taken from a finite set $S$ is denoted by $S^*$.

\subsection{Markov Decision Process}
\label{sub:mdp}

A Markov decision process (MDP) is a finite-state probabilistic system,
where the transition probabilities between the states are determined by 
the control action taken from a given finite set.
Each state of the MDP is labeled by a set of \emph{atomic propositions} 
indicating the properties holding on it, 
e.g., whether the state is a safe/goal state.

\begin{definition} \label{def:mdp}
A labeled Markov decision process (MDP) is a tuple
$\m = (\mss, \mas, \mts, \aps, \ml)$
where
\begin{itemize}
	\item $\mss$ is a finite set of {\em states}.
	\item $\mas$ is a finite set of {\em actions}.
	\item $\mts: \mss \times \mas \times \mss \rightarrow [0, 1]$ is a partial {\em transition probability} function. For any state $\ms \in \mss$ and any action $\ma \in \mas$,
	\[
		\sum_{\ms' \in \mss} \mts(\ms, \ma, \ms') =
		\begin{cases}
			0, &\text{ if } \ma \text{ is not allowed on } \ms\\
			1, &\text{ otherwise.}
		\end{cases}
	\]
	With a slight abuse of notation, let $\mas(\ms)$ be the set of allowed actions on the state $\ms$.
	\item $\aps$ is a finite set of {\em labels}.
	\item $\ml : \mss \rightarrow 2^\aps$ is a {\em labeling function}.
\end{itemize}
\end{definition}

\begin{definition} \label{def:policy}
A policy $\Pi: \mss^* \rightarrow \mas$ decides the action to take 
from the sequence of states visited so far.
Given a policy $\Pi$ and an initial state $\ms \in \mss$, the MDP $\m$ becomes
purely probabilistic, denoted by $\m_{\Pi, \ms}$.
The system $\m_{\Pi, \ms}$ is not necessarily Markovian.
\end{definition}

\subsection{Probabilistic Computation Tree Logic}
\label{sub:pctls}

The probabilistic computation tree logic (PCTL) 
is defined inductively from atomic propositions,
temporal operators and probability operators.
It reasons about the probabilities of time-dependent properties.

\begin{definition}[Syntax]
\label{def:syntax}
Let $\aps$ be a set of atomic propositions.
A \pctls state formula is defined by
\[
\begin{split}
	\phi ::= & \ap
	\mid \neg \phi
	\mid \phi_1 \land \phi_2
	\mid \lp^{\min}_{\Join p} (\lnext \phi)
	\mid \lp^{\max}_{\Join p} (\lnext \phi)
	\\ & 
	\mid \lp^{\min}_{\Join p} (\phi_1 \luntil_T \phi_2)
	\mid \lp^{\max}_{\Join p} (\phi_1 \luntil_T \phi_2)
	\\ &
	\mid \lp^{\min}_{\Join p} (\phi_1 \lrelease_T \phi_2)
	\mid \lp^{\max}_{\Join p} (\phi_1 \lrelease_T \phi_2)
\end{split}
\]
where $\ap \in \aps$, $\Join \in \{<,>,\leq,\geq\}$, 
$T \in \nat \cup \{ \infty \}$ is a (possibly infinite) time horizon,
and $p \in [0,1]$ is a threshold.\footnote{This logic is a fraction of PCTL$^*$ from~\cite{baier_PrinciplesModelChecking_2008}.
}
The operators $\lp^{\min}_{\Join p}$ and $\lp^{\max}_{\Join p}$ are called probability operators, and the ``next'', ``until'' and ``release'' operators $\lnext$, $\luntil_T$, $\lrelease_T$ are called temporal operators.
\end{definition}
 
More temporal operators can be derived by composition:
for example,
``or'' is $\phi_1 \lor \phi_2 \Coloneqq \neg ( \neg \phi_1 \land \neg \phi_2 )$;
``true'' is $\ltrue = \ap \lor (\neg \ap)$;  
``finally'' is $\lfinal_T \phi \Coloneqq \ltrue \luntil_T \phi$; and
``always'' is $\lglobal_T \phi \Coloneqq \lfalse \lrelease_T \phi$.
For simplicity, we write $\luntil_\infty$, $\lrelease_\infty$, $\lfinal_\infty$ and $\lglobal_\infty$ as $\luntil$, $\lrelease$, $\lfinal$ and $\lglobal$, respectively.

\begin{definition}[Semantics]
\label{def:semantics}
For an MDP $\m = (\mss, \mas, \mts, \allowbreak \mi,   \aps,  \ml)$,
the satisfaction relation $\models$ is defined by for a state $\ms$ or path $\sigma$ by
\[
\begin{split}
   & s \models \ap  \textrm{ iff } \ap \in \ml(s),
\\ & s \models \neg \phi  \textrm{ iff } s \not\models \phi,
\\ & s \models \phi_1 \land \phi_2  \textrm{ iff } s \models \phi_1 \textrm{ and } s \models \phi_2,
\\ & s \models \lp^{\min}_{\Join p} (\lnext \phi)  \textrm{ iff } \min_\Pi \pr_{\sigma \sim \m_{\Pi, s}} \big[ \sigma \models \lnext \phi \big] \Join p,
\\ & s \models \lp^{\max}_{\Join p} (\lnext \phi) \textrm{ iff } \max_\Pi \pr_{\sigma \sim \m_{\Pi, s}} \big[ \sigma \models \lnext \phi \big] \Join p,
\\ & s \models \lp^{\min}_{\Join p} (\phi_1 \luntil_T \phi_2)  \textrm{ iff } \min_\Pi \pr_{\sigma \sim \m_{\Pi, s}} \big[ \sigma \models \phi_1 \luntil_T \phi_2 \big] \Join p,
\\ & s \models \lp^{\max}_{\Join p} (\phi_1 \luntil_T \phi_2)  \textrm{ iff } \max_\Pi \pr_{\sigma \sim \m_{\Pi, s}} \big[ \sigma \models \phi_1 \luntil_T \phi_2 \big] \Join p,
\\ & s \models \lp^{\min}_{\Join p} (\phi_1 \lrelease_T \phi_2)  \textrm{ iff } \min_\Pi \pr_{\sigma \sim \m_{\Pi, s}} \big[ \sigma \models \phi_1 \lrelease_T \phi_2 \big] \Join p,
\\ & s \models \lp^{\max}_{\Join p} (\phi_1 \lrelease_T \phi_2)  \textrm{ iff } \max_\Pi \pr_{\sigma \sim \m_{\Pi, s}} \big[ \sigma \models \phi_1 \lrelease_T \phi_2 \big] \Join p,
\\ & \sigma \models \lnext \phi \textrm{ iff } \sigma(1) \models \phi,
\\ & \sigma \models \phi_1 \luntil_T \phi_2 \textrm{ iff } \exists i \leq T. \ \sigma(i) \models \phi_2 \land \big( \forall j < i. \ \sigma(i) \models \phi_1 \big),
\\ & \sigma \models \phi_1 \lrelease_T \phi_2 \textrm{ iff } \sigma \not\models  \neg \phi_1 \luntil_T \neg \phi_2 
\end{split}
\]
where $\Join \in \{<,>,\leq,\geq\}$. And $\sigma \sim \m_{\Pi, s}$ means the path $\sigma$ is drawn from the MDP $\m$ under the policy $\Pi$, starting from the state $s$ from.
\end{definition}

The PCTL formulas $s \models \lp^{\max}_{\Join p} (\lnext \phi)$ (or $s \models \lp^{\min}_{\Join p} (\lnext \phi)$) mean that the maximal (or minimal) satisfaction probability of ``next'' $\phi$ is $\Join p$.  
The PCTL formulas $s \models \lp^{\max}_{\Join p} (\phi_1 \luntil_T \phi_2)$ (or $s \models \lp^{\min}_{\Join p} (\phi_1 \luntil_T \phi_2)$) mean that the maximal (or minimal) satisfaction probability that $\phi_1$ holds ``until'' $\phi_2$ holds is $\Join p$.

\section{Non-Nested PCTL Specifications}
\label{sec:reach_by_q}

In this section, we consider the statistical model checking of non-nested PCTL specifications using an upper-confidence-bound based Q-learning. 
For simplicity, we focus on $\lp^{\max}_{\Join p} (\ap_1 \luntil_T \ap_2)$ and $\lp^{\max}_{\Join p} (\ap_1 \lrelease_T \ap_2)$ where $\ap_1$ and $\ap_2$ are atomic propositions.
Other cases can be handled in the same way.
We discuss the case of $T=1$ in \cref{sub:one-step}, 
the case of $T > 1$ in \cref{sub:finite_reach},
and the case of $T = \infty$ in \cref{sub:infinite_reach}.
Similar to other works on statistical model checking~\cite{larsen_StatisticalModelChecking_2016,agha_SurveyStatisticalModel_2018}, 
we make the following assumption.

\begin{assumption} \label{ass:indifference}
For 
$s \models \lp^{\max}_{\Join p} (\ap_1 \luntil_T \ap_2)$ 
and 
$s \models \lp^{\max}_{\Join p} \allowbreak (\ap_1 \lrelease_T \ap_2)$ 
with 
$T \in \nat \cup \{\infty\}$
and
$\Join \in \{<,>,\leq,\geq\}$,
we assume that 
$\max_\Pi \pr_{\sigma \sim \m_{\Pi, s}}  \big[  \sigma \models \phi_1 \luntil_T \phi_2 \allowbreak  \big] \neq p$ 
and 
$\max_\Pi \pr_{\sigma \sim \m_{\Pi, s}}  \big[  \sigma \models \phi_1 \lrelease_T \phi_2 \big] \neq p$, 
respectively.
\end{assumption}

When it holds, as the number of samples increases, the samples will be increasingly concentrated on one side of the threshold $p$ by the Central Limit Theorem. Therefore, a statistical analysis based on the majority of the samples has increasing accuracy.
When it is violated, the samples would be evenly distributed between the two sides of the boundary $p$, regardless of the sample size.
Thus, no matter how the sample size increases, the accuracy of any statistical test would not increase.
Compared to statistical model checking algorithms based on sequential probability ratio tests (SPRT)~\cite{zuliani_StatisticalModelChecking_2015,wang_StatisticalVerificationPCTL_2018}, no assumption on the indifference region is required here.
Finally, by~\cref{ass:indifference}, we have the additional semantic equivalence between the PCTL specifications:
$\lp^{\max}_{< p} \psi \equiv \lp^{\max}_{\leq p} \psi$ and
$\lp^{\max}_{> p} \psi \equiv \lp^{\max}_{\geq p} \psi$; thus, we will not distinguish between them below.

For further discussion, we first identify a few trivial cases.
For $\ms \models \lp^{\max}_{> p} (\ap_1 \luntil_T \ap_2)$, let
\begin{equation} \label{lem:0}
\begin{split}
	& S_0 = \{\ms \in \mss \mid \ap_1 \notin L(s), \ap_2 \notin L(s) \}
\\	& S_1 = \{\ms \in \mss \mid \ap_2 \in L(s) \}.
\end{split}
\end{equation}
Then for any policy $\Pi$,
$\pr_{\sigma \sim \m_{\Pi, s}} \big[ \sigma \models \phi_1 \luntil_T \phi_2 \big] = 0$ if 
$\ms \in S_0$; and 
$\pr_{\sigma \sim \m_{\Pi, s}} \big[ \sigma \models \phi_1 \luntil_T \phi_2 \big] = 1$ if 
$\ms \in S_1$.
The same holds for $\ms \models \lp^{\max}_{> p} (\ap_1 \lrelease_T \ap_2)$
by defining $S_1$ to be the union of end components of the MDP $\m$ labeled by $\ap_2$ (this only requires knowing the topology of $\m$)~\cite{baier_PrinciplesModelChecking_2008}.
In the rest of this section, we focus on handling the nontrivial case $s \in S \backslash (S_0 \cup S_1)$.

\subsection{Single Time Horizon}
\label{sub:one-step}

When $T=1$, for any $s \in S \backslash (S_0 \cup S_1)$, 
the PCTL specification $\ap_1 \luntil_T \ap_2$ (or $\ap_1 \lrelease_T \ap_2$) holds on a random path $\sigma$ starting from the state $s$ if and only if $\sigma(1) \in S_1$,
where $S_0$ and $S_1$ are from \eqref{lem:0}.
Thus, it suffices to learn from samples whether
\begin{equation} \label{eq:bandit}
	\max_{\ma \in \mas(\ms)} Q_1 (\ms, \ma) > p,
\end{equation}
where 
\[
	Q_1 (\ms, \ma) = \pr_{\substack{\sigma(1) \sim T(s, a, \cdot)\\\sigma(0) = s}} \big[ \sigma \models \phi_1 \luntil_1 \phi_2 \big]
\]
and $\sigma(1) \sim T(s, a, \cdot)$ means $\sigma(1)$ is drawn from the transition probability $T(s, a, \cdot)$ for state $s$ and action $a$.
This is an $\abs{\mas(\ms)}$-arm bandit problem; we solve this problem by upper-confidence-bound strategies~\cite{kuleshov_AlgorithmsMultiarmedBandit_2014,bubeck_RegretAnalysisStochastic_2012}.

Specifically, for the iteration $k$, let $N^{(k)} (\ms, \ma, \ms')$ be the number samples for the one-step path $(\ms, \ma, \ms')$, and with a slight abuse of notation, let
\begin{equation} \label{eq:update_N}
	N^{(k)} (\ms, \ma) = \sum_{\ms' \in \mss} N^{(k)} (\ms, \ma, \ms').
\end{equation}
The unknown transition probability function $\mts (\ms, \ma, \ms')$ is estimated by the empirical transition probability function
\begin{equation} \label{eq:update_mts}
	\hat{\mts}^{(k)} (\ms, \ma, \ms') = \begin{cases}
		\frac{N^{(k)} (\ms, \ma, \ms')}{N^{(k)} (\ms, \ma)}, &\text{ if } N^{(k)} (\ms, \ma) > 0, \\
		\frac{1}{\abs{\mss}}, &\text{ if } N^{(k)} (\ms, \ma) = 0.
	\end{cases}
\end{equation}
And the estimation of $Q_1 (\ms, \ma)$ from the existing $k$ samples is
\begin{equation} \label{eq:1_reward}
	\hat{Q}^{(k)}_1 (\ms, \ma) = \sum_{s' \in S_1} \hat{\mts}^{(k)} (\ms, \ma, \ms').
\end{equation}
Since the value of the Q-function $Q_1(\ms, \ma) \in [0,1]$ is bounded, we can construct a confidence interval for the estimate $\hat{Q}^{(k)}_1$ with statistical error at most $\delta$ using Hoeffding's inequality by
\begin{equation} \label{eq:1_ul_q}
\begin{split}
		& \underline{Q}^{(k)}_1 (\ms, \ma) = \max \bigg\{ \hat{Q}^{(k)}_1 (\ms, \ma) -
		\sqrt{\frac{\abs{\ln (\delta / 2)}}{2 N^{(k)}(\ms, \ma)}}, 0\bigg\},
 		\\ & \overline{Q}^{(k)}_1 (\ms, \ma) = \min \bigg\{ \hat{Q}^{(k)}_1 (\ms, \ma) +
		\sqrt{\frac{\abs{\ln (\delta / 2)}}{2 N^{(k)}(\ms, \ma)}}, 1\bigg\},
\end{split}
\end{equation}
where we set the value of the division to be $\infty$ for $N^{(k)}(\ms, \ma) = 0$.

\begin{remark}
We use Hoeffding's bounds to yield hard guarantees on the statistical error 
of the model checking algorithms.
Tighter bounds like Bernstein's bounds~\cite{mnih_EmpiricalBernsteinStopping_2008}
can also be used, but they only yield asymptotic guarantees on the statistical error.
\end{remark}

The sample efficiency for learning for the bandit problem~\eqref{eq:bandit} depends on the choice of sampling policy, decided from the existing samples.
A provably best solution is to use the Q-learning from~\cite{kuleshov_AlgorithmsMultiarmedBandit_2014,bubeck_RegretAnalysisStochastic_2012}.
Specifically, an upper confidence bound (UCB) is constructed for each state-action pair using the number of samples and the observed reward, and the best action is chosen with the highest possible reward, namely the UCB.
The sampling policy is chosen by maximizing the possible reward greedily:
\begin{equation} \label{eq:1_next}
	\pi_1^{(k)} (\ms) = \text{argmax}_{\ma \in \mas(\ms)} \overline{Q}^{(k)}_1 (\ms, \ma).
\end{equation}
The action is chosen arbitrarily when there are multiple candidates.
The choice of $\pi_1^{(k)}$ in~\eqref{eq:1_next} ensures that the policy giving the upper bound of the value function gets most frequently sampled in the long run.

To initialize the iteration, the Q-function is set to
\begin{equation} \label{eq:1_q_init}
	\overline{Q}_1^{(0)} (\ms, \ma) =
	\begin{cases}
		1, & \text{ if } s \notin S_0, \\
		0, & \text{ otherwise}, \\
	\end{cases}
	\quad
	\underline{Q}_1^{(0)} (\ms, \ma) =
	\begin{cases}
		1, & \text{ if } s \in S_1, \\
		0, & \text{ otherwise}, \\
	\end{cases}
\end{equation}
to ensure that every state-action is sampled at least once.
The termination condition of the above algorithm is
\begin{equation} \label{eq:1_stop}
	\begin{cases}
		\textrm{true}, &\text{ if } \max_{\ma \in \mas(\ms)} \underline{Q}^{(k)}_1 (\ms) > p, \\
		\textrm{false}, &\text{ if } \max_{\ma \in \mas(\ms)} \overline{Q}^{(k)}_1 (\ms) < p, \\
		\textrm{continue}, &\text{ otherwise,} 
	\end{cases}
\end{equation}
where $p$ is the probability threshold in the non-nested PCTL formula.

\begin{remark}
	For $\ms \models \lp^{\max}_{< p} (\ap_1 \luntil_1 \ap_2)$ or $\ms \models \lp^{\max}_{< p} \allowbreak (\ap_1 \lrelease_1 \ap_2)$, it suffices to change the termination condition~\eqref{eq:1_stop} by returning true if
	$\overline{Q}^{(k)}_1 (\ms) < p$, and returning false if
	$\underline{Q}^{(k)}_1 (\ms) > p$.
	The same statements hold for general PCTL specifications, as discussed in \cref{sub:finite_reach,sub:infinite_reach}
\end{remark}

Now, we summarize the above discussion by \cref{alg:1} and \cref{thm:1} below.

\begin{algorithm}[!t]
	\caption{SMC of $\ms \models \lp^{\max}_{> p} (\ap_1 \luntil_1 \ap_2)$ or $\ms \models \lp^{\max}_{> p} (\ap_1 \lrelease_1 \ap_2)$}
	\label{alg:1}
	\begin{algorithmic}[1]
	\Require MDP $\m$, parameter $\delta$.

	\State Initialize the Q-function, and the policy by~\eqref{eq:1_q_init}\eqref{eq:1_next}.

	\State Obtain $S_0$ and $S_1$ by \eqref{lem:0}.

	\While{True}
	
	\State Sample from $\m$, and update the transition probability function by~\eqref{eq:update_N}\eqref{eq:update_mts}.
	
	\State Update the bounds on the Q-function and the policies by~\eqref{eq:1_ul_q}\eqref{eq:1_next}.
	
	\State Check termination condition~\eqref{eq:1_stop}.

	\EndWhile
	\end{algorithmic}
\end{algorithm}

\begin{theorem} \label{thm:1}
	The return value of \cref{alg:1} is correct with probability at least $1 - \abs{\mas} \delta$.
\end{theorem}

\begin{proof}
    We provide the proof of a more general statement in \cref{thm:finite}.
\end{proof}

\begin{remark}\label{rem:conservative}
The Hoeffding bounds in~\eqref{eq:1_ul_q} are conservative.
Consequently, as shown in the simulations in~\cref{sec:simulation}, the actual 
statistical error of the our algorithms can be smaller than the given value.
However, as the MDP is unknown, finding tighter bounds is challenging.
One possible solution is to use asymptotic bounds, such as Bernstein's bounds~\cite{mnih_EmpiricalBernsteinStopping_2008}.
Accordingly, the algorithm will only give asymptotic probabilistic guarantees.
\end{remark}

\subsection{Finite Time Horizon}
\label{sub:finite_reach}

When $T \in \nat$, for any $s \in S \backslash (S_0 \cup S_1)$, let
\begin{equation} \label{eq:vq}
\begin{split}
	& V_h(s) = \max_{\Pi} \pr_{\sigma \sim \m_{\Pi, s}} (\sigma \models \ap_1 \luntil_h \ap_2),
	\\ & Q_h (s, a) = \max_{\Pi(s) = a} \pr_{\sigma \sim \m_{\Pi, s}} (\sigma \models \ap_1 \luntil_h \ap_2), \quad h \in [T],
\end{split}
\end{equation}
i.e., $V_h(\ms)$ and $Q_h(\ms, \ma)$ are the maximal satisfaction probability of $\ap_1 \luntil_h \ap_2$ 
for a random path starting from $\ms$ for any policy and any policy with first action being $\ma$, respectively.
By definition, $V_h(\ms)$ and $Q_h(\ms, \ma)$ satisfy the Bellman equation
\begin{equation} \label{eq:finite_induct}
\begin{split}
	& V_h(s) = \max_{\ma \in \mas} Q_h (s, a),
	\\ & Q_{h+1} (s, a) = \sum_{s' \in S} \mts(s, a, s') V_h(s') 
	\\ & \qquad = \sum_{s \in S \backslash (S_0 \cup S_1)} \mts(s, a, s') V_h(s') + \sum_{s' \in S_1} \mts(s, a, s').
\end{split}	
\end{equation}
The second equality of the second equation is derived from 
\[
	V_h(s) = \begin{cases}
		0, & \textrm{ if } s \in S_0, \\
		1, & \textrm{ if } s \in S_1,
	\end{cases}
\]
by the semantics of PCTL.

From~\eqref{eq:finite_induct}, we check $\lp^{\max}_{> p} (\ap_1 \luntil_h \ap_2)$ by induction on the time horizon $T$.
For $h \in T$, the lower and upper bounds for $Q_h(\ms, \ma)$ can be derived using the bounds on the value function for the previous step ---
for $h = 1$ from~\eqref{eq:1_ul_q} and for $h > 0$ by the following lemma.
\begin{equation} \label{eq:update_bounds_Q}
\begin{split}
	\underline{Q}^{(k)}_{h+1} (\ms, \ma)
	= & \max \bigg\{0, \sum_{s \in S \backslash (S_0 \cup S_1)} \hat{\mts}^{(k)} (s, a, s') \underline{V}_h(s') + 
	\\ & \qquad \sum_{s' \in S_1} \hat{\mts}^{(k)} (s, a, s') - \sqrt{\frac{\abs{\ln (\delta_h / 2)}}{2 N^{(k)}(\ms, \ma)}}\bigg\},
	\\ \overline{Q}^{(k)}_{h+1} (\ms, \ma)
	= & \max \bigg\{1, \sum_{s \in S \backslash (S_0 \cup S_1)} \hat{\mts}^{(k)} (s, a, s') \underline{V}_h(s') + 
	\\ & \qquad \sum_{s' \in S_1} \hat{\mts}^{(k)} (s, a, s') + \sqrt{\frac{\abs{\ln (\delta_h / 2)}}{2 N^{(k)}(\ms, \ma)}}\bigg\},
\end{split}
\end{equation}
and 
\begin{equation} \label{eq:update_bounds_V}
	\overline{V}^{(k)}_h (\ms) = \max_{\ma \in \mas(\ms)} \overline{Q}^{(k)}_h (\ms, \ma), \quad \underline{V}^{(k)}_h (\ms) = \max_{\ma \in \mas(\ms)} \underline{Q}^{(k)}_h (\ms, \ma),
\end{equation}
where $\delta_h$ is a parameter such that
$Q_h (\ms, \ma) \in [\underline{Q}^{(k)}_h (\ms, \ma), \allowbreak \overline{Q}^{(k)}_h (\ms, \ma)]$ with probability at least $1 - \delta_h$.
The bounds in~\eqref{eq:update_bounds_Q} are derived from~\eqref{eq:finite_induct} by applying Hoeffding's inequality, using the fact that $\mathbb{E} [\hat{\mts}^{(k)} (s, a, s')] = \mts (s, a, s')$ and the Q-functions are bounded within $[0,1]$.

From the boundedness of $Q_h (\ms, \ma) \in [0,1]$, we note that this confidence interval encompasses the statistical error in both the estimated transition probability function $\hat{\mts}^{(k)} (\ms, \ma, \ms')$ and the bounds $\overline{V}^{(k)}_{h} (\ms, \ma)$ and $\underline{V}^{(k)}_{h} (\ms, \ma)$ of the value function.
Accordingly, the policy $\pi_h^{(k)}$ chosen by the OFU principle at the $h$ step is
\begin{equation} \label{eq:update_policy}
	\pi_h^{(k)} (\ms) = \text{argmax}_{\ma \in \mas(\ms)} \overline{Q}^{(k)}_h (\ms, \ma),
\end{equation}
with an optimal action chosen arbitrarily when there are multiple candidates,
to ensure that the policy giving the upper bound of the value function is sampled the most in the long run.
To initialize the iteration, the Q-function is set to
\begin{equation} \label{eq:q_init}
	\overline{Q}_h^{(0)} (\ms, \ma) =
	\begin{cases}
		1, & \text{ if } s \notin S_0 \\
		0, & \text{ otherwise,} \\
	\end{cases}
	\quad
	\underline{Q}_h^{(0)} (\ms, \ma) =
	\begin{cases}
		1, & \text{ if } s \in S_1 \\
		0, & \text{ otherwise,} \\
	\end{cases}
\end{equation}
for all $h \in [T]$, to ensure that every state-action is sampled at least once.

Sampling by the updated policy $\pi_h^{(k)} (\ms)$ can be performed in either episodic or non-episodic ways~\cite{szepesvari_AlgorithmsReinforcementLearning_2010}.
The only requirement is that the state-action pair $(\ms, \pi_h^{(k)} (\ms))$ should be performed frequently for each $h \in [T]$ and for each state $s$ satisfying $s \in S \backslash (S_0 \cup S_1)$.
In addition, batch samples may be drawn, namely sampling over the state-action pairs multiple times before updating the policy.
In this work, for simplicity, we use a non-episodic, non-batch sampling method, by drawing
\begin{equation} \label{eq:draw}
	\ms' \sim \mts(\ms, \pi_h^{(k)} (\ms), \cdot),
\end{equation}
for all $h \in [T]$ and state $s$ such that $\ap_1 \in L(s), \ap_2 \notin L(s)$.
The Q-function and the value function are set and initialized by~\eqref{eq:update_bounds_V} and~\eqref{eq:q_init}.
The termination condition is give by
\begin{equation} \label{eq:stop}
\lp^{\max}_{> p} \phi:
\begin{cases}
	\text{false}, &\text{ if } \overline{V}^{(k)}_H (\ms_0) < p, \\
	\text{true}, &\text{ if } \underline{V}^{(k)}_H (\ms_0) > p, \\
	\textrm{continue}, &\text{ otherwise,} 
\end{cases}
\end{equation}
where $p$ is the probability threshold in the non-nested PCTL formula.
The above discussion is summarized by \cref{alg:finite} and \cref{thm:finite}.

\begin{algorithm}[!t]
\caption{SMC of $\ms \models \lp^{\max}_{> p} (\ap_1 \luntil_T \ap_2)$ or $\ms \models \lp^{\max}_{> p} (\ap_1 \lrelease_T \ap_2)$}
\label{alg:finite}
\begin{algorithmic}[1]
\Require MDP $\m$, parameters $\delta_h$ for $h \in [T]$.

\State Initialize the Q-function and the policy by~\eqref{eq:q_init}\eqref{eq:update_policy}.

\State Obtain $S_0$ and $S_1$ by \eqref{lem:0}.

\While{true}

\State Sample by~\eqref{eq:draw}, and update the transition probability function by~\eqref{eq:update_N}\eqref{eq:update_mts}.

\State Update the bounds by~\eqref{eq:update_bounds_Q}\eqref{eq:update_bounds_V} and the policy by~\eqref{eq:update_policy}.

\State Check the termination condition~\eqref{eq:stop}.

\EndWhile
\end{algorithmic}
\end{algorithm}

\begin{theorem} \label{thm:finite}
\cref{alg:finite} terminates with probability $1$ and its return value is correct with probability at least $1 - N \abs{\mas} \sum_{h \in [T]} {\delta_h}$, where $N = \abs{S \backslash (S_0 \cup S_1)}$.
\end{theorem}

\begin{proof}
By construction, as the number of iterations $k \rightarrow \infty$, $\overline{V}^{(k)}_T - \underline{V}^{(k)}_T \rightarrow 0$. Thus, by~\cref{ass:indifference}, the termination condition~\eqref{eq:stop} will be satisfied with probability $1$.
Now, let $\texttt{E}$ be the event that the return value of \cref{alg:finite} is correct, and let $\texttt{F}_k$ be the event that~\cref{alg:finite} terminates at the iteration $k$, then we have $\pr(\texttt{E}) = \sum_{k \in \nat} \pr(\texttt{E} \vert \texttt{F}_k) \pr(\texttt{F}_k)$.
For any $k$, the event $\texttt{E}$ happens given that $\texttt{F}_k$ holds,
if the Hoeffding confidence intervals given by~\eqref{eq:update_bounds_Q} hold for any actions $\ma \in \mas$, $h \in [T]$, and state $s$ with $s \in S \backslash (S_0 \cup S_1)$.
Thus, we have $\pr(\texttt{E} \vert \texttt{F}_k) \geq 1 - N \abs{\mas} \sum_{h \in [T]} {\delta_h}$, where $N = \abs{S \backslash (S_0 \cup S_1)}$,
implying that the return value of \cref{alg:finite} is correct with probability $\pr(\texttt{E}) \geq 1 - N \abs{\mas} \sum_{h \in [T]} {\delta_h}$.
\end{proof}

By~\cref{thm:finite}, the desired overall statistical error splits into the statistical errors for each state-action pair through the time horizon.
For implementation, we can split it equally by $\delta_1 = \cdots = \delta_H$.
The specification $\lp^{\min}_{\Join p}(\phi)$ can be handled by replacing $\text{argmax}$ with $\text{argmin}$ in~\eqref{eq:update_policy}, and $\max$ with $\min$ in~\eqref{eq:update_bounds_V}. 
The termination condition is the same as~\eqref{eq:stop}.

\begin{remark}\label{rem:disprove}
Due to the semantics in~\cref{def:semantics}, running \cref{alg:finite}  proving $\lp^{\max}_{> p}(\phi)$ or disproving $\lp^{\max}_{< p}(\phi)$ is easier than disproving $\lp^{\max}_{> p}(\phi)$ or proving $\lp^{\max}_{< p}(\phi)$; and the difference increases with the number of actions $\abs{\mas}$ and the time horizon $T$.
This is because proving $\lp^{\max}_{> p}(\phi)$ or disproving $\lp^{\max}_{< p}(\phi)$ requires only finding and evaluating some policy $\Pi$ with $\pr_{\m_\Pi} [\ms \models \phi] > p$, while disproving it requires evaluating all possible policies with sufficient accuracy.
This is illustrated by the simulation results presented in~\cref{sec:simulation}.
\end{remark}

\subsection{Infinite Time Horizon}
\label{sub:infinite_reach}

Infinite-step satisfaction probability can be estimated from finite-step satisfaction probabilities, using the monotone convergence of the value function in the time step $H$,
\begin{equation} \label{eq:monotone}
	V_0 (\ms) \leq \ldots \leq V_H (\ms) \leq \ldots \leq V (\ms) = \lim_{H \rightarrow \infty} V_H (\ms).
\end{equation}
Therefore, if the satisfaction probability is larger than $p$ for some step $H$, then the statistical model checking algorithm should terminate, namely,
\begin{equation} \label{eq:stop_infinite}
\begin{cases}
	\text{false}, &\text{ if } \underline{V}^{(k)}_H (\ms_0) > p, \\
	\text{continue}, &\text{ otherwise,}
\end{cases}
\end{equation}
where $p$ is the probability threshold in the non-nested PCTL formula.

The general idea in using the monotonicity to check infinite horizon satisfaction probability in finite time is that if we check both $\lp^{\max}_{> p} (\ap_1 \luntil \ap_2)$ and its negation $\lp^{\min}_{> 1-p} (\neg \ap_1 \lrelease \neg \ap_2)$ at the same time, one of them should terminate in finite time.
Here $\neg \ap_1$ and $\neg \ap_2$ are treated as atomic propositions.
We can use~\cref{alg:finite} to check their satisfaction probabilities for any time horizon $T$ simultaneously.
The termination in finite time is guaranteed, if the time horizon for both computations increase with the iterations.
The simplest choice is to increase $H$ by $1$ for every $K$ iterations; however, this brings the problem of tuning $K$.
Here, we use the convergence of the best policy as the criterion for increasing $H$ for each satisfaction computation.
Specifically, for all the steps $h$ in each iteration, in addition to finding the optimal policy $\pi_{h}^{(k)} (\ms)$ with respect to the upper confidence bounds of the Q-functions $\overline{Q}^{(k)}_{h} (\ms, \ma)$ by~\eqref{eq:update_policy},
we also consider the the optimal policy
with respect to the lower confidence bounds of the Q-functions $\underline{Q}^{(k)}_{h} (\ms, \ma)$.
Obviously, when $\pi_h^{(k)} (\ms) \in \text{argmax}_{\ma \in \mas} \underline{Q}^{(k)}_{h} (\ms, \ma)$, we know that the policy $\pi_h^{(k)} (\ms)$ is optimal for all possible Q-functions within $[\underline{Q}^{(k)}_h, \overline{Q}^{(k)}_h]$.
This implies that these bounds are fine enough for estimating $Q_H$; thus, if the algorithm does not terminate by the condition~\eqref{eq:stop_infinite}, we let
\begin{equation} \label{eq:H}
	H \gets \begin{cases}
		1, &\text{initially},
		\\ H + 1, & \text{ if } \pi_H^{(k)} (\ms) \in \text{argmax}_{\ma \in \mas} \underline{Q}^{(k)}_{H} (\ms, \ma)
		\\ & \quad \text { for all } \ms \in \mss, \\
		\text{Continue}, &\text{ otherwise}.
	\end{cases}
\end{equation}
Combining the above procedure, we derive \cref{alg:infinite} and \cref{thm:infinite} below for statistically model checking PCTL formula $\lp^{\max}_{> p} (\ap_1 \luntil \ap_2)$.

\begin{algorithm}[!t]
\caption{SMC of $\ms \models \lp^{\max}_{> p} (\ap_1 \luntil \ap_2)$}
\label{alg:infinite}
\begin{algorithmic}[1]
\Require MDP $\m$, parameters $\delta_h$ for $h \in \nat$.

\State Initialize two sets of Q-function and the policy by~\eqref{eq:q_init}\eqref{eq:update_policy} for \textbf{(i)} $\lp^{\max}_{> p} (\ap_1 \luntil \ap_2)$ and \textbf{(ii)} $\lp^{\min}_{> 1-p} (\neg \ap_1 \lrelease \neg \ap_2)$, respectively.

\State Obtain $S_0$ and $S_1$ for \textbf{(i)} and \textbf{(ii)} respectively by \eqref{lem:0}.

\While{True}

\State Sample by~\eqref{eq:draw}, and update $\hat{\mts}^{(k)} (\ms, \ma, \ms')$ by~\eqref{eq:update_N}\eqref{eq:update_mts}.

\State Update the bounds on the Q-function, the policies, the value function, and the time horizon by~\eqref{eq:update_bounds_Q}\eqref{eq:update_policy}\eqref{eq:update_bounds_V}\eqref{eq:H} respectively for \textbf{(i)} and \textbf{(ii)}.

\State Check the termination condition~\eqref{eq:stop_infinite}.

\EndWhile

\end{algorithmic}
\end{algorithm}

\begin{theorem} \label{thm:infinite}
\cref{alg:infinite} terminates with probability $1$ and its return value is correct with probability 
$1 - \abs{\mas} \max \{N_1, N_2\} \sum_{h \in [T]} {\delta_h}$, 
where 
$H$ is the largest time horizon when the algorithm stops,
$N_1 = \abs{S \backslash (S_0^\phi \cup S_1^\phi)}$ and $N_2 = \abs{S \backslash (S_0^\psi \cup S_1^\psi)}$ 
with $S_0^\phi \cup S_1^\phi$ 
and $S_0^\psi \cup S_1^\psi$ 
derived from \eqref{lem:0} 
for $\phi = \ap_1 \luntil \ap_2$ 
and $\psi = \ap_1 \lrelease \ap_2$, 
respectively.
\end{theorem}

\begin{proof}
Terminates with probability $1$ follows easily from~\eqref{eq:monotone}.
Following the proof of~\cref{thm:finite}, if the procedure of checking either $\lp^{\max}_{< p} \phi$ or its negation $\lp^{\min}_{< 1-p} (\neg \phi)$ stops and the largest time horizon is $H$, then the return value is correct
with probability at least $1 - \abs{\mas} \max \{ N_1, N_2\}  \allowbreak \sum_{h \in [T]} {\delta_h}$.
Thus, the theorem holds.
\end{proof}

\begin{remark}
By~\cref{thm:infinite}, given the desired overall confidence level $\delta$, we can split it geometrically by $\delta_h = (1 - \lambda) \lambda^{h-1} \delta$, where $\lambda \in (0,1)$.
\end{remark}

\begin{remark}
Similar to \cref{sub:finite_reach}, checking $\lp^{\min}_{\sim p}(\phi)$ for $\sim \in \{<,>,\leq,\geq\}$ is derived by replacing $\text{argmax}$ with $\text{argmin}$ in~\eqref{eq:update_policy}, and $\max$ with $\min$ in~\eqref{eq:update_bounds_V}. The termination condition is the same as~\eqref{eq:stop_infinite}.
\end{remark}

\begin{remark}
Finally, we note that the exact savings of sample costs for \cref{alg:finite,alg:infinite} depend on the structure of the MDP.
Specifically, the proposed method is more efficient than~\cite{brazdil_VerificationMarkovDecision_2014,fu_ProbablyApproximatelyCorrect_2014,balle_LearningWeightedAutomata_2015},
when the satisfaction probabilities differ significantly among actions,
as it can quickly detect sub-optimal actions without over-sampling on them.
On the other hand, if all the state-action pairs yield the same Q-value, then an equal number of samples will be spent on each of them --- in this case, the sample cost of \cref{alg:finite,alg:infinite} is the same as~\cite{brazdil_VerificationMarkovDecision_2014,fu_ProbablyApproximatelyCorrect_2014,balle_LearningWeightedAutomata_2015}.
\end{remark}

\newcommand{\until}{\luntil}\newcommand{\eventually}{\mathbf{F}}\DeclarePairedDelimiter{\Size}{|}{|}\newcommand{\TR}[1]{\multirow{2}{*}{#1}}\newcommand{\TU}[1]{\multirow{2}{*}{$\approx$~\StrLeft{#1}{6}}}

\section{Simulation} \label{sec:simulation}

To evaluate the performance of the proposed algorithms,
we ran them on two different sets of examples.
In all the simulations, the transition probabilities are unknown to the algorithm (this is different from~\cite{brazdil_VerificationMarkovDecision_2014}).

The first set contains 10 randomly generated MDPs with different sizes.
For these MDPs,
we considered the formula $\lp^{\max}_{< p} (\alpha_1 \allowbreak \until_H \alpha_2)$ for the finite horizon, and
$\lp^{\max}_{< p} (\alpha_1 \until  \alpha_2)$ for the infinite horizon.
In both cases, $\alpha_1$, $\alpha_2$, and $H$ are chosen arbitrarily.
The second set contains $9$ versions of the Sum of Two Dice program.
The standard version~\cite{KY76} models the sum of two fair dice,
each with $6$ different possibilities numbered $1,\ldots,6$.
To consider MDPs with different sizes, we consider $9$ cases where each dice is still fair, but has $n$
possibilities numbered $1,\ldots,n$, with $n=3$ in the smallest example, and $n=17$ in the largest example.
For these MDPs,
we considered the formula $\lp^{\max}_{< p} (\eventually_H\alpha)$ for the finite horizon, and
$\lp^{\max}_{< p} (\eventually\alpha)$ for the infinite horizon.
In both cases $\alpha$ encodes the atomic predicate that is true iff
values of both dice are chosen and their sum is less than an arbitrarily chosen constant.
Also, in the finite case, $H=5$ in the smallest example and $10$ everywhere else.

For each MDP, we first numerically estimate the values of
$\max_{\Pi} \pr_{\m_\Pi} [\alpha_1\luntil_H \allowbreak \alpha_2]$ and
$\max_{\Pi} \pr_{\m_\Pi} [\alpha_1\luntil\alpha_2]$, for the randomly generated MDPs,
and
$\max_{\Pi} \pr_{\m_\Pi} \allowbreak [\eventually_H\alpha]$, and
$\max_{\Pi} \pr_{\m_\Pi} [\eventually\alpha]$,
for the variants of the two-dice examples,
using the known models on
PRISM~\cite{kwiatkowska_PRISMVerificationProbabilistic_2011}.
Then, we use \cref{alg:finite,alg:infinite} on the example models with only knowledge of the topology of the MDPs and without knowing the exact transition probabilities.
For every MDP, we tested each algorithm with two different thresholds $p$, one smaller and the other larger than the estimated probability, to test the proposed algorithms, with $\delta$ set to $5\%$.
We ran each randomly generated test $100$ times and each two-dice variant test $10$ times.
Here we only report average running time and average number of iterations.
All tests returned the correct answers --- this suggests that the Hoeffding's bounds used in the proposed algorithms are conservative (see Remark~\ref{rem:conservative}).
The algorithms are implemented in Scala and ran on Ubuntu 18.04 with i7-8700 CPU 3.2GHz and 16GB memory.

\NewDocumentCommand{\WhiteText}{}{\color{white}\bf\boldmath}

\begin{table}[!t]
\centering
\setlength\tabcolsep{2mm}
\scalebox{0.85}{\begin{tabular}{|r|r|r|r|Hr|Hr|c|}
\hline
\rowcolor{black}
\multicolumn{1}{|c|}{\WhiteText$\Size{\mss}$}&
\multicolumn{1}{c|}{\WhiteText$\Size{\mas}$}&
\multicolumn{1}{c|}{\WhiteText$H$}&
\multicolumn{1}{c|}{\WhiteText Iter.}&
&\multicolumn{1}{c|}{\WhiteText Time (s)}&
&\multicolumn{1}{c|}{\WhiteText$p$}&
\multicolumn{1}{c|}{\WhiteText PRISM Est.}
\\\hline
\TR{ 3}&\TR{3}&  4 &   208.5 &  61.8 &    0.02 &      21.82 & 0.25 & \TU{0.3533334785169} \\
       &      &  4 &  3528.7 & 227.5 &    0.19 &     197.88 & 0.45 &                      \\\hline
\TR{ 4}&\TR{2}&  4 &   171.8 &  47.8 &    0.01 &      11.44 & 0.09 & \TU{0.1934681923968} \\
       &      &  4 &  3671.7 & 221.7 &    0.22 &     221.86 & 0.29 &                      \\\hline
\TR{ 5}&\TR{2}&  4 &   441.7 & 142.1 &    0.04 &      38.80 & 0.05 & \TU{0.1176582400650} \\
       &      &  4 &  4945.5 & 203.3 &    0.42 &     426.59 & 0.21 &                      \\\hline
\TR{10}&\TR{2}&  4 &   544.7 & 134.0 &    0.14 &     151.91 & 0.04 & \TU{0.0965737835649} \\
       &      &  4 &  5193.3 & 198.0 &    1.45 &    1470.65 & 0.19 &                      \\\hline
\TR{15}&\TR{2}&  3 &   873.1 & 243.0 &    0.28 &     300.27 & 0.04 & \TU{0.0946021249956} \\
       &      &  3 &  4216.7 & 204.5 &    1.28 &    1296.87 & 0.19 &                      \\\hline
\TR{20}&\TR{4}&  5 &   337.6 &  93.9 &    0.99 &    1050.93 & 0.12 & \TU{0.2225524587262} \\
       &      &  5 &  9353.3 & 263.9 &   28.06 &   28613.86 & 0.32 &                      \\\hline
\TR{25}&\TR{5}& 10 &   270.5 &  73.7 &    7.57 &    7976.02 & 0.09 & \TU{0.1912732335968} \\
       &      & 10 & 25709.8 & 355.5 &  728.49 &  740817.04 & 0.29 &                      \\\hline
\TR{30}&\TR{5}& 10 &   355.6 &  95.8 &   14.35 &   14943.54 & 0.08 & \TU{0.1731897194020} \\
       &      & 10 & 27161.7 & 394.6 & 1085.77 & 1104590.54 & 0.27 &                      \\\hline
\TR{35}&\TR{5}& 10 &   328.9 &  90.5 &   18.82 &   19775.56 & 0.09 & \TU{0.1826511894781} \\
       &      & 10 & 27369.6 & 336.4 & 1529.84 & 1556336.49 & 0.28 &                      \\\hline
\TR{40}&\TR{5}& 10 &   390.0 & 106.4 &   26.79 &   27983.28 & 0.08 & \TU{0.1674652617042} \\
       &      & 10 & 30948.8 & 416.3 & 2122.24 & 2150728.06 & 0.26 &                      \\\hline
\end{tabular}}
\caption{Checking $\lp^{\max}_{< p}(\alpha_1\until_H\alpha_2)$ on random MDPs.
Atomic propositions $\alpha_1$ and $\alpha_2$ are chosen arbitrarily.
Column ``Iter.'' is the average number of iterations,
Column ``PRISM~Est.'' is PRISM's estimation of the actual probability.}\label{tbl:res-finite-rnd}
\end{table}

\begin{table}[!t]
\centering
\setlength\tabcolsep{2mm}
\scalebox{0.85}{\begin{tabular}{|r|r|r|r|r|r|c|}
\hline
\rowcolor{black}
\multicolumn{1}{|c|}{\WhiteText$n$}&
\multicolumn{1}{c|}{\WhiteText$\Size{\mss}$}&
\multicolumn{1}{c|}{\WhiteText$H$}&
\multicolumn{1}{c|}{\WhiteText Iter.}&
\multicolumn{1}{c|}{\WhiteText Time (s)}&
\multicolumn{1}{c|}{\WhiteText$p$}&
\multicolumn{1}{c|}{\WhiteText PRISM Est.}
\\\hline
\TR{3 } & \TR{   36} & \TR{ 5}&  15.6 &   0.79  & 0.27  & \TU{0.3750000000000}  \\
        &            &        &  39.7 &   1.31  & 0.47  &                       \\\hline
\TR{5 } & \TR{  121} & \TR{10}&  32.9 &   4.07  & 0.20  & \TU{0.3027343750000}  \\
        &            &        &  93.8 &  11.30  & 0.40  &                       \\\hline
\TR{6 } & \TR{  169} & \TR{10}&  29.2 &   4.95  & 0.29  & \TU{0.3955078125000}  \\
        &            &        &  90.4 &  15.33  & 0.49  &                       \\\hline
\TR{7 } & \TR{  196} & \TR{10}&  33.4 &   6.85  & 0.31  & \TU{0.4101562500000}  \\
        &            &        &  70.7 &  13.84  & 0.51  &                       \\\hline
\TR{9 } & \TR{  400} & \TR{10}&  41.1 &  15.87  & 0.21  & \TU{0.3105468750000}  \\
        &            &        & 110.3 &  43.47  & 0.41  &                       \\\hline
\TR{11} & \TR{  529} & \TR{10}&  46.4 &  24.07  & 0.28  & \TU{0.3867187500000}  \\
        &            &        & 111.9 &  58.15  & 0.48  &                       \\\hline
\TR{13} & \TR{  729} & \TR{10}&  25.8 &  18.52  & 0.20  & \TU{0.3046875000000}  \\
        &            &        & 109.9 &  80.28  & 0.40  &                       \\\hline
\TR{15} & \TR{ 1156} & \TR{10}&  42.1 &  47.90  & 0.05  & \TU{0.1025390625000}  \\
        &            &        & 155.3 & 178.29  & 0.20  &                       \\\hline
\TR{17} & \TR{ 1369} & \TR{10}&  24.0 &  32.17  & 0.06  & \TU{0.1328125000000}  \\
        &            &        & 155.0 & 208.78  & 0.23  &                       \\\hline
\end{tabular}}
\caption{Checking $\lp^{\max}_{< p}(\eventually_H\alpha)$ on variants of sum-of-two-dice MDPs.
Atomic proposition $\alpha$ is true iff both dice have chosen their numbers and their sum is less than an arbitrary constant.
Column $n$ is the number of alternatives on each dice.
Each MDP has two actions ({\em i.e.} $|\mas|=2$).
Column ``Iter.'' is the average number of iterations,
Column ``PRISM~Est.'' is PRISM's estimation of the actual probability.}\label{tbl:res-finite-dice}
\end{table}

\cref{tbl:res-finite-rnd,tbl:res-finite-dice} show the results for finite horizon reachability.
An interesting observation in these tables is that in all examples,
disproving the formula is 3 to 100 times faster.
We believe this is mainly because, to disprove $\lp^{\max}_{< p}(\alpha_1\until_H\alpha_2)$,
all we need is {\em one} policy $\Pi$ for which $\lp_{> p}(\alpha_1\until_H\alpha_2)$ holds.
However, to prove $\lp^{\max}_{< p}(\alpha_1\until_H\alpha_2)$, one needs to show that
{\em every} policy $\Pi$ satisfies $\lp_{< p}(\alpha_1\until_H\alpha_2)$
(see Remark~\ref{rem:disprove}).

\begin{table}[!t]
\centering
\setlength\tabcolsep{2mm}
\scalebox{0.85}{\begin{tabular}{|r|r|r|r|r|c|r|r|}
\hline
\rowcolor{black}
\multicolumn{1}{|c|}{\WhiteText$\Size{\mss}$}&
\multicolumn{1}{c|}{\WhiteText$\Size{\mas}$}&
\multicolumn{1}{c|}{\WhiteText Iter.}&
\multicolumn{1}{c|}{\WhiteText Time (s)}&
\multicolumn{1}{c|}{\WhiteText$p$}&
\multicolumn{1}{c|}{\WhiteText PRISM Est.}&
\multicolumn{1}{c|}{\WhiteText$H_1$}&
\multicolumn{1}{c|}{\WhiteText$H_2$}
\\\hline
  \TR{3} & \TR{3} &   126.0 &   0.03&  0.25&  \TU{0.3537848739585}&  14.7&    15.8  \\
         &        &   111.3 &   0.01&  0.45&                      &  12.9&    13.0  \\\hline
  \TR{4} & \TR{2} &   103.7 &   0.01&  0.09&  \TU{0.1934843680899}&  13.0&    27.2  \\
         &        &    73.0 &   0.01&  0.29&                      &   9.6&    19.2  \\\hline
  \TR{5} & \TR{2} &   239.9 &   0.06&  0.05&  \TU{0.1176596683003}&  12.4&    59.3  \\
         &        &    92.7 &   0.00&  0.21&                      &   6.9&    24.3  \\\hline
 \TR{10} & \TR{2} &   891.4 &   0.24&  0.04&  \TU{0.0965759820342}&   3.2&   225.6  \\
         &        &    79.6 &   0.00&  0.19&                      &   1.1&    21.5  \\\hline
 \TR{15} & \TR{2} &  1862.0 &   0.61&  0.04&  \TU{0.0946390484241}&   1.3&   117.8  \\
         &        &   161.3 &   0.01&  0.19&                      &   1.0&    11.0  \\\hline
 \TR{20} & \TR{4} &  1336.2 &   0.27&  0.12&  \TU{0.2225536322037}&   1.0&     1.0  \\
         &        & 16843.8 &   3.02&  0.32&                      &   1.0&     1.8  \\\hline
 \TR{25} & \TR{5} &  1619.2 &   0.64&  0.09&  \TU{0.1912732304949}&   1.0&     1.0  \\
         &        &154621.6 &  56.56&  0.29&                      &   1.0&     2.0  \\\hline
 \TR{30} & \TR{5} &  2246.8 &   1.05&  0.08&  \TU{0.1731897068421}&   1.0&     1.0  \\
         &        & 86617.0 &  42.53&  0.27&                      &   1.0&     1.3  \\\hline
 \TR{35} & \TR{5} &  1925.1 &   0.82&  0.09&  \TU{0.1826511868090}&   1.0&     1.0  \\
         &        & 13219.0 &   5.79&  0.28&                      &   1.0&     1.0  \\\hline
 \TR{40} & \TR{5} &  2401.5 &   1.35&  0.08&  \TU{0.1674652603300}&   1.0&     1.0  \\
         &        & 12158.6 &   6.66&  0.26&                      &   1.0&     1.0  \\\hline
\end{tabular}}
\caption{Checking $\lp^{\max}_{< p}(\alpha_1\until\alpha_2)$ on random MDPs.
Columns $H_1$ and $H_2$ are values of the corresponding variables in \cref{alg:infinite} at termination.}\label{tbl:res-infinite-rnd}
\end{table}

\cref{tbl:res-infinite-rnd,tbl:res-infinite-dice} show the results for infinite horizon reachability.
Note that \cref{alg:infinite} considers both the formula and its negation, and
contrary to the finite horizon reachability, disproving a formula is not always
faster for the infinite case.
In most of the larger examples that are randomly generated, $H_1$ and $H_2$ are very small on average.
This shows that in these examples, the algorithm was smart enough to learn
there is no need to increase $H$ in order to solve the problem.
However, this is not the case for two-dice examples.
We believe this is because in the current implementation, the decision to increase $H$ does not consider
the underlying graph of the MDP. For example, during the execution, if the policy forces the state to enter a self-loop
with only one enabled action, which is the case in two-dice examples, then after every iteration the value of $H$ will
be increased by $1$.

\begin{table}[!t]
\centering
\setlength\tabcolsep{2mm}
\scalebox{0.85}{\begin{tabular}{|r|r|r|r|r|c|r|r|}
\hline
\rowcolor{black}
\multicolumn{1}{|c|}{\WhiteText$n$}&
\multicolumn{1}{c|}{\WhiteText$\Size{\mss}$}&
\multicolumn{1}{c|}{\WhiteText Iter.}&
\multicolumn{1}{c|}{\WhiteText Time (s)}&
\multicolumn{1}{c|}{\WhiteText$p$}&
\multicolumn{1}{c|}{\WhiteText PRISM Est.}&
\multicolumn{1}{c|}{\WhiteText$H_1$}&
\multicolumn{1}{c|}{\WhiteText$H_2$}
\\\hline
\TR{ 3} & \TR{  36} &  191.1 &   3.39  & 0.56 & \TU{0.6666664630175} & 192.1 & 110.9 \\
        &           &  232.6 &   2.82  & 0.76 &                      & 233.6 & 117.4 \\\hline
\TR{ 5} & \TR{ 121} &  359.9 &  15.01  & 0.29 & \TU{0.3999997649680} & 153.6 & 360.9 \\
        &           &  378.6 &  18.15  & 0.49 &                      & 149.0 & 379.6 \\\hline
\TR{ 6} & \TR{ 169} &  268.0 &   8.15  & 0.31 & \TU{0.4166665673256} & 140.2 &  71.9 \\
        &           &  572.4 &  19.88  & 0.51 &                      & 127.5 &  73.6 \\\hline
\TR{ 7} & \TR{ 196} &  291.2 &  17.30  & 0.32 & \TU{0.4285713750869} & 136.3 & 292.1 \\
        &           &  281.1 &  16.42  & 0.52 &                      & 129.7 & 282.0 \\\hline
\TR{ 9} & \TR{ 400} &  545.1 & 109.93  & 0.55 & \TU{0.6543202942270} & 394.6 & 199.2 \\
        &           &  593.8 & 129.73  & 0.75 &                      & 445.9 & 200.2 \\\hline
\TR{11} & \TR{ 529} &  493.4 & 106.65  & 0.44 & \TU{0.5454542201005} & 210.3 & 179.1 \\
        &           &  799.0 & 173.00  & 0.64 &                      & 288.9 & 177.2 \\\hline
\TR{13} & \TR{ 729} &  720.6 & 231.51  & 0.36 & \TU{0.4615382983684} & 116.3 & 276.8 \\
        &           &  393.1 & 118.15  & 0.56 &                      & 127.2 & 300.6 \\\hline
\TR{15} & \TR{1156} & 2027.9 & 1762.51 & 0.36 & \TU{0.4666657858409} & 134.9 & 627.0 \\
        &           & 1374.7 & 833.66  & 0.56 &                      & 155.2 & 489.2 \\\hline
\TR{17} & \TR{1369} & 1995.4 & 1105.89 & 0.37 & \TU{0.4705876478439} & 107.1 & 161.3 \\
        &           & 1469.4 & 767.87  & 0.57 &                      & 100.7 & 165.4 \\\hline
\end{tabular}}
\caption{Checking $\lp^{\max}_{< p}(\eventually\alpha)$ on sum-of-two-dice MDPs.
Columns $H_1$ and $H_2$ are values of the corresponding variables in \cref{alg:infinite} at termination.}\label{tbl:res-infinite-dice}
\end{table}

\section{Conclusion} \label{sec:conclusion}

We proposed a statistical model checking method 
for Probabilistic Computation Tree Logic
on Markov decision processes
using reinforcement learning. 
We first checked PCTL formulas with bounded time horizon, using upper-confidence-bounds based Q-learning,
and then extended the technique to unbounded-time specifications by finding a proper truncation time by checking the specification of interest and its negation at the same time.
Finally, we demonstrated the efficiency of our method on several case studies.

\bibliographystyle{IEEEtran}

\begin{thebibliography}{10}
\providecommand{\url}[1]{#1}
\csname url@samestyle\endcsname
\providecommand{\newblock}{\relax}
\providecommand{\bibinfo}[2]{#2}
\providecommand{\BIBentrySTDinterwordspacing}{\spaceskip=0pt\relax}
\providecommand{\BIBentryALTinterwordstretchfactor}{4}
\providecommand{\BIBentryALTinterwordspacing}{\spaceskip=\fontdimen2\font plus
\BIBentryALTinterwordstretchfactor\fontdimen3\font minus
  \fontdimen4\font\relax}
\providecommand{\BIBforeignlanguage}[2]{{\expandafter\ifx\csname l@#1\endcsname\relax
\typeout{** WARNING: IEEEtran.bst: No hyphenation pattern has been}\typeout{** loaded for the language `#1'. Using the pattern for}\typeout{** the default language instead.}\else
\language=\csname l@#1\endcsname
\fi
#2}}
\providecommand{\BIBdecl}{\relax}
\BIBdecl

\bibitem{baier_ModelcheckingAlgorithmsContinuoustime_2003}
C.~Baier, B.~Haverkort, H.~Hermanns, and J.~Katoen, ``Model-checking algorithms
  for continuous-time {{Markov}} chains,'' \emph{IEEE Transactions on Software
  Engineering}, vol.~29, no.~6, pp. 524--541, 2003.

\bibitem{clarke_HandbookModelChecking_2018}
\emph{\BIBforeignlanguage{en}{Handbook of {{Model Checking}}}}, 2018.

\bibitem{larsen_StatisticalModelChecking_2016}
K.~G. Larsen and A.~Legay, ``\BIBforeignlanguage{en}{Statistical {{Model
  Checking}}: {{Past}}, {{Present}}, and {{Future}}},'' in
  \emph{\BIBforeignlanguage{en}{Leveraging {{Applications}} of {{Formal
  Methods}}, {{Verification}} and {{Validation}}: {{Foundational
  Techniques}}}}, 2016, pp. 3--15.

\bibitem{agha_SurveyStatisticalModel_2018}
G.~Agha and K.~Palmskog, ``A {{Survey}} of {{Statistical Model Checking}},''
  \emph{ACM Trans. Model. Comput. Simul.}, vol.~28, no.~1, pp. 6:1--6:39, 2018.

\bibitem{fainekos_TemporalLogicMotion_2009}
G.~E. Fainekos, A.~Girard, H.~{Kress-Gazit}, and G.~J. Pappas, ``Temporal logic
  motion planning for dynamic robots,'' \emph{Automatica}, vol.~45, no.~2, pp.
  343--352, 2009.

\bibitem{kress-gazit_SynthesisRobotsGuarantees_2018}
H.~{Kress-Gazit}, M.~Lahijanian, and V.~Raman, ``Synthesis for {{Robots}}:
  {{Guarantees}} and {{Feedback}} for {{Robot Behavior}},'' \emph{Annual Review
  of Control, Robotics, and Autonomous Systems}, vol.~1, no.~1, pp. 211--236,
  2018.

\bibitem{bozkurt_ControlSynthesisLinear_2020}
A.~K. Bozkurt, Y.~Wang, M.~Zavlanos, and M.~Pajic, ``Control synthesis from
  linear temporal logic specifications using model-free reinforcement
  learning,'' in \emph{{{IEEE International Conference}} on {{Robotics}} and
  {{Automation}} ({{ICRA}}) (Submitted)}, 2020.

\bibitem{henriques_StatisticalModelChecking_2012}
D.~Henriques, J.~G. Martins, P.~Zuliani, A.~Platzer, and E.~M. Clarke,
  ``Statistical {{Model Checking}} for {{Markov Decision Processes}},'' in
  \emph{2012 {{Ninth International Conference}} on {{Quantitative Evaluation}}
  of {{Systems}}}, 2012, pp. 84--93.

\bibitem{brazdil_VerificationMarkovDecision_2014}
T.~Br{\'a}zdil, K.~Chatterjee, M.~Chmel{\'i}k, V.~Forejt, J.~K{\v
  r}et{\'i}nsk{\'y}, M.~Kwiatkowska, D.~Parker, and M.~Ujma,
  ``\BIBforeignlanguage{en}{Verification of {{Markov Decision Processes Using
  Learning Algorithms}}},'' in \emph{\BIBforeignlanguage{en}{Automated
  {{Technology}} for {{Verification}} and {{Analysis}}}}, 2014, pp. 98--114.

\bibitem{fu_ProbablyApproximatelyCorrect_2014}
J.~Fu and U.~Topcu, ``Probably {{Approximately Correct MDP Learning}} and
  {{Control With Temporal Logic Constraints}},'' \emph{arXiv:1404.7073 [cs]},
  2014.

\bibitem{ashok_PACStatisticalModel_2019}
P.~Ashok, J.~K{\v r}et{\'i}nsk{\'y}, and M.~Weininger, ``{{PAC Statistical
  Model Checking}} for {{Markov Decision Processes}} and {{Stochastic
  Games}},'' \emph{arXiv:1905.04403 [cs]}, 2019.

\bibitem{sutton_ReinforcementLearningIntroduction_2018}
R.~S. Sutton and A.~G. Barto, ``\BIBforeignlanguage{en}{Reinforcement
  {{Learning}}: {{An Introduction}}},'' p. 352, 2018.

\bibitem{li_ReinforcementLearningTemporal_2016}
X.~Li, C.-I. Vasile, and C.~Belta, ``Reinforcement {{Learning With Temporal
  Logic Rewards}},'' \emph{arXiv:1612.03471 [cs]}, 2016.

\bibitem{jones_RobustSatisfactionTemporal_2015}
A.~Jones, D.~Aksaray, Z.~Kong, M.~Schwager, and C.~Belta, ``Robust
  {{Satisfaction}} of {{Temporal Logic Specifications}} via {{Reinforcement
  Learning}},'' \emph{arXiv:1510.06460 [cs]}, 2015.

\bibitem{littman_EnvironmentIndependentTaskSpecifications_2017}
M.~L. Littman, U.~Topcu, J.~Fu, C.~Isbell, M.~Wen, and J.~MacGlashan,
  ``Environment-{{Independent Task Specifications}} via {{GLTL}},''
  \emph{arXiv:1704.04341 [cs]}, 2017.

\bibitem{baier_PrinciplesModelChecking_2008}
C.~Baier and J.-P. Katoen, \emph{Principles of Model Checking}, 2008.

\bibitem{hahn_OmegaRegularObjectivesModelFree_2019a}
E.~M. Hahn, M.~Perez, S.~Schewe, F.~Somenzi, A.~Trivedi, and D.~Wojtczak,
  ``\BIBforeignlanguage{en}{Omega-{{Regular Objectives}} in {{Model}}-{{Free
  Reinforcement Learning}}},'' in \emph{\BIBforeignlanguage{en}{Tools and
  {{Algorithms}} for the {{Construction}} and {{Analysis}} of {{Systems}}}},
  2019, vol. 11427, pp. 395--412.

\bibitem{hasanbeig_ReinforcementLearningTemporal_2019}
M.~Hasanbeig, Y.~Kantaros, A.~Abate, D.~Kroening, G.~J. Pappas, and I.~Lee,
  ``Reinforcement learning for temporal logic control synthesis with
  probabilistic satisfaction guarantees,'' in \emph{Proceedings of the 58th
  Conference on Decision and Control}.\hskip 1em plus 0.5em minus 0.4em\relax
  IEEE, 2019, pp. 5338--5343.

\bibitem{hahn_LazyProbabilisticModel_2013}
E.~M. Hahn, G.~Li, S.~Schewe, A.~Turrini, and L.~Zhang, ``Lazy {{Probabilistic
  Model Checking}} without {{Determinisation}},'' \emph{arXiv:1311.2928 [cs]},
  2013.

\bibitem{DBLP:conf/fsttcs/KiniV17}
D.~Kini and M.~Viswanathan, ``Complexity of model checking mdps against {LTL}
  specifications,'' in \emph{37th {IARCS} Annual Conference on Foundations of
  Software Technology and Theoretical Computer Science, {FSTTCS} 2017, December
  11-15, 2017, Kanpur, India}, ser. LIPIcs, S.~V. Lokam and R.~Ramanujam, Eds.,
  vol.~93.\hskip 1em plus 0.5em minus 0.4em\relax Schloss Dagstuhl -
  Leibniz-Zentrum f{\"{u}}r Informatik, 2017, pp. 35:1--35:13.

\bibitem{kini2015limit}
------, ``Limit deterministic and probabilistic automata for ltl$\setminus$
  gu,'' in \emph{International Conference on Tools and Algorithms for the
  Construction and Analysis of Systems}.\hskip 1em plus 0.5em minus 0.4em\relax
  Springer, 2015, pp. 628--642.

\bibitem{sen_StatisticalModelChecking_2004}
K.~Sen, M.~Viswanathan, and G.~Agha, ``\BIBforeignlanguage{en}{Statistical
  {{Model Checking}} of {{Black}}-{{Box Probabilistic Systems}}},'' in
  \emph{\BIBforeignlanguage{en}{Computer {{Aided Verification}}}}, 2004, pp.
  202--215.

\bibitem{sen_StatisticalModelChecking_2005}
------, ``\BIBforeignlanguage{en}{On {{Statistical Model Checking}} of
  {{Stochastic Systems}}},'' in \emph{\BIBforeignlanguage{en}{Computer {{Aided
  Verification}}}}, 2005, vol. 3576, pp. 266--280.

\bibitem{szepesvari_AlgorithmsReinforcementLearning_2010}
C.~Szepesv{\'a}ri, ``\BIBforeignlanguage{en}{Algorithms for {{Reinforcement
  Learning}}},'' \emph{\BIBforeignlanguage{en}{Synthesis Lectures on Artificial
  Intelligence and Machine Learning}}, vol.~4, no.~1, pp. 1--103, 2010.

\bibitem{zuliani_StatisticalModelChecking_2015}
P.~Zuliani, ``\BIBforeignlanguage{en}{Statistical model checking for biological
  applications},'' \emph{\BIBforeignlanguage{en}{International Journal on
  Software Tools for Technology Transfer}}, vol.~17, no.~4, pp. 527--536, 2015.

\bibitem{wang_StatisticalVerificationPCTL_2018}
Y.~Wang, N.~Roohi, M.~West, M.~Viswanathan, and G.~E. Dullerud, ``Statistical
  verification of {{PCTL}} using stratified samples,'' in \emph{6th {{IFAC
  Conference}} on {{Analysis}} and {{Design}} of {{Hybrid Systems}} ({{ADHS}}),
  {{IFAC}}-{{PapersOnLine}}}, vol.~51, 2018, pp. 85--90.

\bibitem{kuleshov_AlgorithmsMultiarmedBandit_2014}
V.~Kuleshov and D.~Precup, ``Algorithms for multi-armed bandit problems,''
  \emph{arXiv:1402.6028 [cs]}, 2014.

\bibitem{bubeck_RegretAnalysisStochastic_2012}
S.~Bubeck, ``\BIBforeignlanguage{en}{Regret {{Analysis}} of {{Stochastic}} and
  {{Nonstochastic Multi}}-armed {{Bandit Problems}}},''
  \emph{\BIBforeignlanguage{en}{Foundations and Trends\textregistered{} in
  Machine Learning}}, vol.~5, no.~1, pp. 1--122, 2012.

\bibitem{mnih_EmpiricalBernsteinStopping_2008}
V.~Mnih, C.~Szepesv{\'a}ri, and J.-Y. Audibert,
  ``\BIBforeignlanguage{en}{Empirical {{Bernstein}} stopping},'' in
  \emph{\BIBforeignlanguage{en}{Proceedings of the 25th International
  Conference on {{Machine}} Learning - {{ICML}} '08}}, 2008, pp. 672--679.

\bibitem{balle_LearningWeightedAutomata_2015}
B.~Balle and M.~Mohri, ``\BIBforeignlanguage{en}{Learning {{Weighted
  Automata}}},'' in \emph{\BIBforeignlanguage{en}{Algebraic {{Informatics}}}},
  2015, vol. 9270, pp. 1--21.

\bibitem{KY76}
D.~Knuth and A.~Yao, \emph{Algorithms and Complexity: New Directions and Recent
  Results}.\hskip 1em plus 0.5em minus 0.4em\relax Academic Press, 1976, ch.
  The complexity of nonuniform random number generation.

\bibitem{kwiatkowska_PRISMVerificationProbabilistic_2011}
M.~Kwiatkowska, G.~Norman, and D.~Parker, ``\BIBforeignlanguage{en}{{{PRISM}}
  4.0: {{Verification}} of {{Probabilistic Real}}-{{Time Systems}}},'' in
  \emph{\BIBforeignlanguage{en}{Computer {{Aided Verification}}}}, 2011, vol.
  6806, pp. 585--591.

\end{thebibliography}

\end{document}